\RequirePackage{letltxmacro}
\RequirePackage{etoolbox}
\makeatletter
\let\ORIlabel\label
\let\ORIrefstepcounter\refstepcounter
\AddToHook{package/hyperref/before}{
	\let\label\ORIlabel
	\let\refstepcounter\ORIrefstepcounter
}
\makeatother

\documentclass[onefignum,onetabnum,pagebackref,dvipsnames]{siamart220329}

\usepackage{mathtools, amsfonts, amssymb, mathrsfs,stmaryrd}
\usepackage{enumerate,enumitem}
\usepackage{subcaption}
\usepackage{soul}
\usepackage{algorithm, algpseudocode,algorithmicx}
\usepackage{comment}
\usepackage{xspace}
\usepackage{mleftright}
\usepackage{cite}
\usepackage{booktabs}

\algrenewcommand\algorithmiccomment[1]{\hfill\textcolor{Mulberry}{$\triangleright$ #1}}

\renewcommand*{\backref}[1]{}
\renewcommand*{\backrefalt}[4]{%
	\ifcase #1 %
	(No citations.)
	\or
	(Cited on page #2.)
	\else
	(Cited on pages #2.)
	\fi
}

\usepackage[most]{tcolorbox}

\newcommand{\real}{\mathbb{R}}
\newcommand{\complex}{\mathbb{C}}

\DeclareMathOperator{\diag}{diag}
\DeclareMathOperator{\orth}{orth}
\newcommand{\mat}[1]{\boldsymbol{#1}}
\renewcommand{\vec}[1]{\boldsymbol{#1}}
\newcommand{\lowrank}[1]{\mleft\llbracket #1 \mright\rrbracket}
\newcommand{\norm}[1]{\mleft\| #1 \mright\|}

\newcommand{\QR}{\textsf{QR}\xspace}

\newcommand{\evec}{\mathbf{e}}

\newcommand{\Id}{\mathbf{I}}

\DeclareMathOperator{\expect}{\mathbb{E}}
\DeclareMathOperator{\prob}{\mathbb{P}}

\newcommand{\order}{\mathcal{O}}

\newcommand{\set}[1]{\mathsf{#1}}

\renewcommand{\hat}[1]{\widehat{#1}}

\newcommand{\Ahat}{\smash{\mat{\hat{A}}}}
\DeclareMathOperator{\Vol}{Vol}
\newcommand{\kDPP}[1]{\operatorname{DPP}_{#1}}
\newcommand{\VS}[1]{\operatorname{VS}_{#1}}
\newcommand{\RejectionRPQR}{\textsf{RejectionRPQR}\xspace}
\newcommand{\ARP}{\textsf{ARP}\xspace}
\newcommand{\ProjARP}{\textsf{ProjARP}\xspace}
\newcommand{\SkARP}{\textsf{SkARP}\xspace}
\newcommand{\RPQR}{\textsf{RPQR}\xspace}
\newcommand{\SkQR}{\textsf{SkQR}\xspace}

\usepackage{amsfonts}
\usepackage{graphicx}
\usepackage{epstopdf}
\ifpdf
\DeclareGraphicsExtensions{.eps,.pdf,.png,.jpg}
\else
\DeclareGraphicsExtensions{.eps}
\fi

\newsiamremark{remark}{Remark}
\newsiamthm{claim}{Claim}

\newcommand{\mytitle}{Adaptive randomized pivoting and volume sampling}

\headers{Adaptive randomized pivoting and volume sampling}{E. N. Epperly}

\title{\mytitle\thanks{Date: November 3, 2025.
		\funding{The author thanks the Miller Institute for Basic Research in Science, University of California Berkeley for supporting this work.
			This work was initiated while the author was at Caltech, supported under aegis of Joel Tropp by ONR Award N00014-24-1-2223 and the Caltech Carver Mead New Adventures Fund.}}}

\author{Ethan N. Epperly\thanks{Department of Mathematics, University of California Berkeley, Berkeley, CA 94720 USA (\email{eepperly@berkeley.edu}, \url{https://ethanepperly.com}).}}

\usepackage{amsopn}

\ifpdf
\hypersetup{
  pdftitle={Adaptive randomized pivoting and volume sampling},
  pdfauthor={Ethan N. Epperly}
}
\fi

\begin{document}

\maketitle

\begin{abstract}
Adaptive randomized pivoting (\ARP) is a recently proposed and highly effective algorithm for column subset selection.
This paper reinterprets the \ARP algorithm by drawing connections to the volume sampling distribution and active learning algorithms for linear regression.
As consequences, this paper presents new analysis for the \ARP algorithm and faster implementations using rejection sampling.
\end{abstract}



\section{Introduction}

The problem of selecting a subset of columns or rows that approximately span a given matrix is classical in computational linear algebra and scientific computing.
This task has gained renewed attention in machine learning as the \emph{column subset selection problem}.
Applications include interpretable data analysis \cite{MD09a}, feature selection \cite{BMD08}, experimental design \cite{DPPL24,DCMW19}, rank-structured matrix computations \cite{Mar11,Wil21}, and tensor network algorithms \cite{OT10,TSL24b}.
Classically, column subset selection was solved by (partial) column-pivoted \QR decomposition \cite[sec.~5.4.2]{GV13} or, for better accuracy at higher cost, strong rank-revealing \QR factorization \cite{GE94}.
Over the past three decades, \emph{randomized} approaches for this problem have been studied, including squared column norm sampling \cite{FKV98}, leverage score sampling \cite{Woo14,BMD09}, adaptive sampling/randomly pivoted \QR \cite{DRVW06,DV06,CETW,ETW24a}, volume sampling \cite{DRVW06,DW17,DW18a}, and sketchy pivoting \cite{VM17,DCMP23,DM23}.

\subsection{Adaptive randomized pivoting}
A recent paper of Cortinovis and Kressner \cite{CK24} introduced a new strategy called \emph{adaptive randomized pivoting} (\ARP).
Here is the basic idea.
Suppose we wish to sample a representative set of \emph{rows} of a matrix $\mat{A}\in\complex^{m\times n}$.
As input, \ARP requires an orthonormal basis $\mat{Q}\in\complex^{m\times k}$ which approximates the range of $\mat{A}$, i.e., $\norm{\mat{A} - \mat{Q}\mat{Q}^* \mat{A}}_{\rm F}\approx 0$.
As usual, ${}^*$ is the conjugate transpose.
To select a subset $\set{S}$ of $k$ rows, \ARP performs a randomly pivoted \QR decomposition on $\mat{Q}^*$ (see \cref{sec:arp} for details).
Having chosen the row set $\set{S}$, \ARP produces one of two \emph{low-rank approximations} to the matrix $\mat{A}$, either
\begin{equation} \label{eq:arp-lra}
	\Ahat_1 \coloneqq \mat{Q}\mat{Q}(\set{S},:)^{-1}\cdot\mat{A}(\set{S},:) \quad \text{or} \quad \Ahat_2 \coloneqq \mat{A}\mat{A}(\set{S},:)^\dagger \cdot \mat{A}(\set{S},:).
\end{equation}
A low-rank approximation of the form $\mat{W}\cdot\mat{A}(\set{S},:)$ are called an \emph{XR decomposition} or \emph{row interpolative decomposition}.
These approximations have applications in rank-structured matrix computations \cite{Mar11,Wil21} and tensor network algorithms \cite{OT10,TSL24b}.
We review \ARP more in \cref{sec:arp}.
In some applications, the low-rank approximation $\mat{W} \cdot \mat{A}(\set{S},:)$ is primary output of interest.
In other contexts, such as genetics or feature selection, we care more about the subset $\set{S}$ itself, which indicates a small ``core set'' of interesting genes or features from our data set.

Cortinovis and Kressner's main theoretical result \cite[Thm.~2.1]{CK24} shows that \ARP produces near-optimal row subsets:

\begin{theorem}[Adaptive randomized pivoting] \label{thm:arp}
	The low-rank approximations \cref{eq:arp-lra} produced by \ARP satisfy
	\begin{equation} \label{eq:arp-bound}
		\expect \norm{\mat{A} - \Ahat_2}_{\rm F}^2 \le \expect \norm{\mat{A} - \Ahat_1}_{\rm F}^2 = (k+1) \norm{(\Id - \mat{Q}\mat{Q}^* )\mat{A}}_{\rm F}^2.
	\end{equation}
	In particular, if $\mat{Q}$ consists of the $k$ dominant left singular vectors of $\mat{A}$, then $\mat{Q}\mat{Q}^* \mat{A} = \lowrank{\mat{A}}_k$ is the optimal rank-$k$ approximation to $\mat{A}$ and
	\begin{equation} \label{eq:arp-near-opt}
		\expect \norm{\mat{A} - \Ahat_2}_{\rm F}^2 \le \expect \norm{\mat{A} - \Ahat_1}_{\rm F}^2 = C_k \norm{\mat{A} - \lowrank{\mat{A}}_k}_{\rm F}^2 \quad \text{with }C_k = k+1.
	\end{equation}
\end{theorem}

Observe that since $\Ahat_2$ is the orthogonal projection of $\mat{A}$ onto the row span of $\mat{A}(\set{S},:)$, the matrix $\Ahat_2$ achieves the minimum Frobenius norm approximation error for any approximation spanned in the row span of $\mat{A}(\set{S},:)$.
In particular, 
\begin{equation} \label{eq:comparision-A1-A2}
	\norm{\mat{A} - \Ahat_2}_{\rm F} \le \norm{\mat{A} - \Ahat_1}_{\rm F}.
\end{equation}
As such, the main content of \cref{thm:arp} is the equality statements. 

\Cref{thm:arp} is striking because \cref{eq:arp-near-opt} matches the optimal \emph{existence result} for a rank-$k$ approximation to a matrix spanned by $k$ columns.
That is, no interpolative decomposition can achieve $C_k < k+1$ on a worst-case matrix $\mat{A}$ \cite[Prop.~3.3]{DRVW06}.

\subsection{Contributions and outline}

This paper draws a connection between the \ARP method and theory and algorithms for volume sampling.
Specifically, the subset $\set{S}$ in \ARP is shown to be a sample from the volume sampling distribution \cite{DRVW06}
\begin{equation*}
	\prob \{\set{S} = \set{T}\} = \frac{\Vol(\set{T})^2}{\sum_{|\set{R}| = k} \Vol(\set{R})^2} \quad \text{where } \Vol(\set{T}) \coloneqq |{\det(\mat{Q}(\set{T},:))}|.
\end{equation*}
%
Using this connection, we can rederive \cref{thm:arp} from known results from the volume sampling literature.
This connection is developed in \cref{eq:vs}.

In addition to yielding a new interpretation of adaptive randomized pivoting, the connection between volume sampling and adaptive randomized pivoting yields efficient rejection-sampling based implementations of \ARP, which we develop in \cref{eq:fast-arp}.
\Cref{sec:skarp} proposes a second way of accelerating \ARP using the \emph{oversampled sketchy interpolative decomposition} approach of \cite{DCMP23}, and \cref{sec:end-to-end} contains an end-to-end error analysis of \ARP with sketching.
Experiments in \cref{sec:experiments} demonstrate that our fast implementations can accelerate \ARP by over an order of magnitude, making \ARP methods among the fastest and most accurate strategies for row subset selection.

\section{Adaptive randomized pivoting} \label{sec:arp}

Let us now introduce the \ARP algorithm more systematically.
The \ARP algorithm takes as input an orthonormal matrix $\mat{Q}$ for which $\norm{\mat{A} - \mat{Q}\mat{Q}^* \mat{A}}_{\rm F}\approx 0$. 
Finding such a $\mat{Q}$ is the \emph{range-finder problem} in randomized linear algebra literature \cite{HMT11}.
The simplest solution, suggested for use in \ARP by Cortinovis and Kressner, is to first \emph{sketch} the matrix $\mat{A}$
\begin{equation} \label{eq:randomized-rangefinder}
	\mat{B} \coloneqq \mat{A}\mat{\Omega}
\end{equation}
using a \emph{random embedding matrix} $\mat{\Omega}$ \cite[secs.~8--9]{MT20}, then orthonormalize $\mat{Q} = \orth(\mat{B})$.
To improve this estimate, we can apply subspace iteration $\mat{B} \coloneqq (\mat{A}\mat{A}^* )^q\mat{A}\mat{\Omega}$ or block Krylov iteration \cite{TW23,Gu15,HMT11}.
In settings where we are interested in finding the best-possible subset of rows and are less concerned with computational cost, we can simply compute an SVD of $\mat{A}$ and choose $\mat{Q}$ to be the $k$ dominant left singular vectors.

Now, we describe the row sampling step.
The \emph{randomly pivoted \QR method} (originally introduced as \emph{adaptive sampling}) \cite{CETW,DV06,DRVW06,ETW24a} selects a subset of $k$ columns of a matrix $\mat{M}\in\complex^{d\times m}$ as follows:
For $i=1,\ldots,k$,
\begin{enumerate}
	\item \textit{Sample} a random column $s_i$ from the \emph{squared column norm} distribution
	\begin{subequations} \label{eq:rpqr}
		\begin{equation}
			\label{eq:rpqr-sample}
			\prob \{s_i = j\} = \norm{\mat{M}(:,j)}^2/\norm{\mat{M}}_{\rm F}^2.
		\end{equation}
		\item \textit{Update} the matrix $\mat{M}$ by orthogonalizing against the selected column:
		\begin{equation} \label{eq:rpqr-orthog}
			\mat{M} \gets \left( \Id - \mat{M}(:,s_i)\mat{M}(:,s_i)^\dagger \right) \mat{M}.
		\end{equation}
	\end{subequations}
\end{enumerate}
Simple modifications of this procedure output a rank-$k$ approximation to $\mat{M}$ or a \QR decomposition of the selected submatrix $\mat{M}(:,\set{S})$ for $\set{S} = \{s_1,\ldots,s_k\}$.
The randomly pivoted \QR method was introduced as the \emph{adaptive sampling} algorithm by Deshpande, Rademacher, Vempala, and Wang \cite{DRVW06,DV06}.
My collaborators and I revisited this algorithm in \cite{CETW,ETW24a}, established the connection with pivoted \QR decompositions, and suggested the name \emph{randomly pivoted \QR}; see also \cite{DCMP23}.

To select a row subset in their algorithm, Cortinovis and Kressner run the randomly pivoted \QR algorithm on $\mat{Q}^* $.
Pseudocode for \ARP appears in \cref{alg:arp}.

\begin{algorithm}[t]
	\caption{Adaptive randomized pivoting \cite{CK24}}
	\label{alg:arp}
	\begin{algorithmic}[1]
		\Require Matrix $\mat{A} \in \complex^{m \times n}$, subset size $k$, interpolation type $\mathtt{type}\in\{1,2,\mathrm{OSID}\}$
		\Ensure Subset $\set{S} \subseteq \{1,\ldots,m\}$ and interpolation matrix $\mat{W} \in \complex^{m\times k}$
		\State $\mat{\Omega} \gets \text{$n\times k$ random embedding}$ \Comment{E.g., SparseStack (\cref{def:sparsestack})}
		\State $\mat{Q} \gets \Call{orth}{\mat{A}\mat{\Omega}}$ \Comment{Range-finder}
		\State $\set{S} \gets \Call{RandomlyPivotedQR}{\mat{Q}^*}$ \Comment{\cref{eq:rpqr} or \cref{alg:rejectionrpqr}}
		\If{$\mathtt{type} = 1$}
		\State $\mat{W} \gets \mat{Q}\mat{Q}(\set{S},:)^{-1}$ \Comment{Apply inverse stably}
		\ElsIf{$\mathtt{type} = 2$}
		\State $\mat{W} \gets \mat{A}\mat{A}(\set{S},:)^\dagger$ \Comment{Apply pseudoinverse stably, e.g., by \QR}
		\ElsIf{$\mathtt{type} = \mathrm{OSID}$} \Comment{See \cref{sec:skarp}}
		\State $\mat{\Phi} \gets \text{$n\times ck$ random embedding}$ \Comment{E.g., SparseStack (\cref{def:sparsestack}), $c=2$}
		\State $\mat{W} \gets (\mat{A}\mat{\Phi})(\mat{A}(\set{S},:)\mat{\Phi})^\dagger$ \Comment{Apply pseudoinverse stably, e.g., by \QR}
		\EndIf
	\end{algorithmic}
\end{algorithm}

\section{Connection between \ARP and volume sampling} \label{eq:vs}

The \ARP algorithm has strong connections to the volume sampling and determinantal point processes lurking just beneath the surface.
We begin by defining these distributions \cite{DM21a,DRVW06}.

\begin{definition}[Volume sampling and $k$-DPPs]
	Fix a matrix $\mat{B} \in \complex^{m\times n}$ and an integer $k \le \min(m,n)$.
	A random subset $\set{S} \subseteq \{1,\ldots,m\}$ of size $k$ is said to follow the \emph{volume sampling distribution}, written $\set{S} \sim \VS{k}(\mat{B})$, if it satisfies
	\begin{equation*}
		\prob \{\set{S} = \set{T}\} = \frac{\Vol(\set{T})^2}{\sum_{|\set{R}| = k} \Vol(\set{R})^2}.
	\end{equation*}
    Here, the volume $\Vol(\set{T})$ is defined as the product of the singular values of $\mat{B}(\set{T},:)$.
	Given a Hermitian positive semidefinite matrix $\mat{H} \in \complex^{m\times m}$, a \emph{$k$-DPP} is a random $k$-element subset $\set{S} \subseteq \{1,\ldots,m\}$, written $\set{S} \sim \kDPP{k}(\mat{H})$, with distribution
	\begin{equation*}
		\prob \{\set{S} = \set{T} \} = \frac{\det \mat{H}(\set{T},\set{T})}{\sum_{|\set{R}| = k} \det \mat{H}(\set{R},\set{R})}.
	\end{equation*}
	The subset $\set{S}$ is said to be a \emph{projection DPP} if $\mat{H}$ is a rank-$k$ orthoprojector.
\end{definition}

The volume sampling and $k$-DPP distributions are closely linked: For a matrix $\mat{B} \in \complex^{m\times n}$ and $k \le \min(m,n)$, we have $\VS{k}(\mat{B}) = \kDPP{k}(\mat{B}\mat{B}^* )$.
That is, volume sampling on $\mat{B}$ is $k$-DPP sampling on the Gram matrix $\mat{B}\mat{B}^* $.

\subsection{Volume sampling and linear regression}

Derezi\'nski, Warmuth, and collaborators investigated the use of volume sampling for solving \emph{active linear regression} problems \cite{DW17,DW18a}.
Here is the idea.
Suppose we are interested in fitting a linear model $\vec{x} \mapsto \vec{x}^\top \vec{\beta}$ to input--output data $\{(\vec{x}_i,y_i) : i=1,\ldots,m\} \subseteq \complex^k \times \complex$.
We may solve this problem as a linear least-squares problem.
First, package the inputs $\vec{x}_i$ as rows of a matrix $\mat{X}$ and outputs $y_i$ as entries of a vector $\vec{y}$. Then, determine the coefficients $\vec{\beta}$ by solving the optimization problem
\begin{equation} \label{eq:ls}
	\operatorname*{minimize}_{\vec{\beta} \in \complex^k} \norm{\mat{X}\vec{\beta} - \vec{y}}^2.
\end{equation}
If $\mat{X}$ has full rank, the unique solution is  $\vec{\beta} = \mat{X}^\dagger\vec{y}$ and
\begin{equation} \label{eq:optimal-ls}
	\min_{\vec{\beta} \in \complex^k} \norm{\mat{X}\vec{\beta} - \vec{y}}^2 = \norm{\smash{(\Id - \mat{X}\mat{X}^\dagger)\vec{y}}}^2.
\end{equation}

Now, consider an \emph{active learning} variant of this problem.
We are now given input data points $\{\vec{x}_i\}_{i=1}^m$ and a budget to read a \emph{subset} of the output values $\{y_i\}_{i=1}^m$.
Specifically, we assume that we are free to access the values $\vec{y}(\set{S})$ of a subset of $|\set{S}| = \ell$ chosen output values; the remaining $m - \ell$ values remain a mystery to us, and we must use the collected values $\vec{y}(\set{S})$ to produce an approximate solution $\vec{\hat{\beta}}$ to the least squares problem \cref{eq:ls}.
Volume sampling provides a mathematically elegant solution to this problem.
For the case where $\ell = k$, we have this result \cite[Thm.~5, Prop.~7]{DW17}:

\begin{theorem}[Active linear regression by volume sampling] \label{thm:active-regression}
	Suppose $\mat{X}\in\complex^{m\times k}$ has full-rank and draw $k$ points $\set{S} \sim \VS{k}(\mat{X})$.
	Define an approximate solution
	\begin{equation} \label{eq:adaptive-linear-approx-solution}
		\vec{\hat{\beta}} \coloneqq \mat{X}(\set{S},:)^{-1}\vec{y}(\set{S}).
	\end{equation}
	Then $\vec{\hat{\beta}}$ is unbiased $\expect[\vec{\hat{\beta}}] = \vec{\beta}$ and satisfies
	\begin{equation} \label{eq:vs-guarantee}
		\expect \norm{\smash{\mat{X}\vec{\hat{\beta}} - \vec{y}}}^2 = \expect \norm{\smash{\mat{X}\mat{X}(\set{S},:)^{-1}\vec{y}(\set{S}) - \vec{y}}}^2 = (k+1) \norm{\smash{(\Id - \mat{X}\mat{X}^\dagger) \vec{y}}}^2.
	\end{equation}
	Recall that $\norm{\smash{(\Id - \mat{X}\mat{X}^\dagger) \vec{y}}}^2$ is the minimum least-square deviation \cref{eq:optimal-ls}.
\end{theorem}

By accessing the vector $\vec{y}$ at $k$ points chosen using the volume sampling distribution, we achieve a near-optimal least-square deviation \cref{eq:optimal-ls}.
In expectation, the suboptimality is a factor $k+1$.
We note that the original proofs treated real variables, but the proof transfers to the complex case without issue.
On a worst-case instance $(\mat{X},\vec{y})$, the factor of $k+1$ in this result is optimal for any active linear regression method; see \cref{app:optimaity}.

\subsection{Connection to \ARP}

The connection between \ARP and volume sampling is beginning to suggest itself.
The final link is provided by the following result: 
\begin{theorem}[Randomly pivoted \QR and volume sampling] \label{thm:rpqr-volume}
	Let $\mat{Q} \in \complex^{m\times k}$ have orthonormal columns.
	Then the pivot set $\set{S}$ selected by randomly pivoted \QR on $\mat{Q}^*$ is a sample from the volume sampling distribution $\set{S} \sim \VS{k}(\mat{Q})$ or, equivalently, the projection DPP distribution $\set{S} \sim \kDPP{k}(\mat{Q}\mat{Q}^*)$.
\end{theorem}
This result appears in a general way in \cite[Thm.~18]{HKPV06}; see also \cite[sec.~2.1]{BTA23} for a particularly clean explanation of this result.

Let us now discover the connection between \ARP and volume sampling.
Constructing a low-rank approximation $\mat{Q}\mat{F} \approx \mat{A}$ may be seen as a fitting problem
\begin{equation} \label{eq:matrix-fitting}
	\operatorname*{minimize}_{\mat{F} \in \complex^{k\times n}} \norm{\mat{Q}\mat{F} - \mat{A}}_{\rm F}^2.
\end{equation}
To construct an \emph{interpolative} low-rank approximation, we wish to find an approximate solution to this fitting problem after reading as few rows of $\mat{A}$ as possible.
The \ARP algorithm outputs a low-rank approximation of the form
\begin{equation*}
	\Ahat_1 = \mat{Q}\mat{F} \quad \text{for } \mat{F} = \mat{Q}(\set{S},:)^{-1}\mat{A}(\set{S},:).
\end{equation*}
Observe that $\mat{F} = \mat{Q}(\set{S},:)^{-1}\mat{A}(\set{S},:)$ can be interpreted as the active linear regression solution \cref{eq:adaptive-linear-approx-solution} to the matrix fitting problem \cref{eq:matrix-fitting}, and the subset $\set{S}$ selected by \ARP is a sample from the volume sampling distribution $\set{S} \sim \VS{k}(\mat{Q})$.
Thus, in this way, \ARP can be seen as equivalent to active linear regression with volume sampling.

Using this observation, Cortinovis and Kressner's main result \cref{thm:arp} follows immediately from \cref{thm:active-regression}.

\begin{proof}[Proof of \cref{thm:arp}]
	By \cref{thm:rpqr-volume}, $\set{S}$ produced in the \ARP algorithm is a sample from $\VS{k}(\mat{Q})$.
	To prove the theorem, we unpack the squared Frobenius norm column-by-column, apply \cref{thm:active-regression}, and repackage:
	\begin{align*}
		\expect \norm{\mat{A} - \Ahat_1}_{\rm F}^2 &= \expect \norm{\mat{A} - \mat{Q}\mat{Q}(:,\set{S})^{-1}\mat{A}(\set{S},:)}_{\rm F}^2 = \sum_{j=1}^n \expect \norm{\mat{A}(:,j) - \mat{Q}\mat{Q}(:,\set{S})^{-1}\mat{A}(\set{S},j)}_{\rm F}^2 \\
		&= \sum_{j=1}^n (k+1) \norm{(\Id - \mat{Q}\mat{Q}^*)\mat{A}(:,j)}_{\rm F}^2 = (k+1) \norm{(\Id - \mat{Q}\mat{Q}^*)\mat{A}}_{\rm F}^2. \hfill \qed
	\end{align*}
\end{proof}

One can also run this argument in the opposite direction: Cortinovis and Kressner's proof of \cref{thm:arp} also gives an alternate proof of the active linear regression identity \cref{eq:vs-guarantee}.
Indeed, the volume sampling distribution $\VS{k}(\mat{X})$ is invariant under right-multiplication by a nonsingular matrix, so $\VS{k}(\mat{X}) = \VS{k}(\mat{Q})$ for $\mat{Q} = \orth(\mat{X})$.
Invoking  \cref{thm:arp} with $\mat{A} = \vec{y}$ and \cref{thm:rpqr-volume}  establishes \cref{eq:vs-guarantee}.

\section{Fast \ARP by rejection sampling} \label{eq:fast-arp}

When implemented using the single-pass randomized rangefinder \cref{eq:randomized-rangefinder} with an appropriate random embedding (e.g., a SparseStack---see \cref{def:sparsestack}), the runtime of adaptive randomized pivoting for a dense matrix $\mat{A}$ is roughly $\order(mn + mk^2)$ operations:
Applying a fast, structured embedding to a dense matrix $\mat{A}$ requires roughly $\order(mn)$ operations, and computing $\mat{Q} = \orth(\mat{B})$, running randomly pivoted \QR on $\mat{Q}^*$, and forming the product $\mat{W} = \mat{Q}\mat{Q}(\set{S},:)^{-1}$ each expend $\order(mk^2)$ work.
Using Cortinovis and Kressner's implementation, the dominant cost is the randomly pivoted \QR step, as their randomly pivoted \QR implementation is sequential and relies on vector--vector and matrix--vector arithmetic.
By contrast, computing $\mat{Q} = \orth(\mat{B})$ and $\mat{W} = \mat{Q}\mat{Q}(\set{S},:)^{-1}$ are faster because they use matrix--matrix operations. 

Fortunately, the literature on randomly pivoted \QR, volume sampling, and DPP sampling has developed faster methods.
For running randomly pivoted \QR on a general matrix, myself and coauthors developed the \emph{accelerated randomly pivoted \QR algorithm} \cite{ETW24a}, which uses \emph{rejection sampling} to produce the same set of random pivots as ordinary randomly pivoted \QR but in a block-wise fashion using matrix--matrix arithmetic.
For the present context where we wish to apply randomly pivoted \QR to a matrix $\mat{Q}^*$ with \emph{orthonormal rows}, we have access to an even faster rejection sampling-based randomly pivoted \QR implementation, originally due to \cite{DCMW19} and rediscovered in \cite{BTA23}.
We will call this fastest algorithm \RejectionRPQR.

\subsection{\RejectionRPQR}

Before discussing efficient block implementations, let us first describe a conceptual implementation of \RejectionRPQR.
Suppose we have already sampled columns $s_1,\ldots,s_i$ following the same distribution as the \RPQR procedure.
We seek to draw a new pivot with distribution
\begin{equation} \label{eq:next-pivot}
	\prob \{ s_{i+1} = j \mid s_1,\ldots,s_i\} = \frac{\norm{(\Id - \mat{\Pi}_i) \mat{Q}^*(j,:))}^2}{k - i},
\end{equation}
where $\mat{\Pi}_i$ denotes the orthoprojector onto the column span of $\mat{Q}^*(:,\{s_1,\ldots,s_i\})$.
Sampling from this distribution directly is difficult, as it requires us to orthogonalize the entire matrix $\mat{Q}^*$ against the already-selected columns $\mat{Q}^*(:,\{s_1,\ldots,s_i\})$.
Rejection sampling yields a faster solution.
It proceeds as follows:
\begin{enumerate}
	\item \textbf{Propose.} Sample a proposal $t$ from the \emph{leverage score distribution}
	\begin{equation} \label{eq:lev}
		\prob \{ t = j \} = \ell_j/k \quad \text{where } \ell_j \coloneqq \norm{\mat{Q}(j,:)}^2.
	\end{equation}
	\item  \textbf{Accept?} With probability $\norm{(\Id - \mat{\Pi}_i) \mat{Q}^*(t,:))}^2/\ell_t$, \textbf{\textit{accept}} and set $s_{i+1} \gets t$.
	Otherwise, \textbf{\textit{reject}} and go to step 1.
\end{enumerate}
A short computation verifies that this procedure produces a sample with distribution \cref{eq:next-pivot} upon termination.
The advantage of rejection sampling over the direct implementation \cref{eq:rpqr} is that we only have to orthogonalize the selected columns against a one proposal column each rejection sampling loop, rather than the \emph{entire matrix} every randomly pivoted \QR iteration.

Theoretical analysis of \RejectionRPQR is beautiful and simple \cite{BTA23,DCMW19}.
The upshot is that \RejectionRPQR produces $\set{S} \sim \VS{k}(\mat{Q}) = \kDPP{k}(\mat{Q}\mat{Q}^*)$ after an expected $\order(k\log k)$ rejection sampling steps and $\order(k^3\log k)$ arithmetic operations.
There is also an upfront cost of $\order(mk)$ operations to compute the leverage scores $\{\ell_j\}_{j=1}^m$.

\subsection{Efficient block \RejectionRPQR}

For efficient implementation, we make two modifications to the basic \RejectionRPQR algorithm.
First, we make proposals in blocks of size $k$ to take advantage of block-wise matrix arithmetic.
Second, to facilitate efficient and stable orthogonalization, we maintain a Householder \QR decomposition of the submatrix $\mat{Q}^*(:,\set{S})$.

We describe the block implementation first.
Suppose that we have currently accepted $i < k$ pivots $\set{S} = \{s_1,\ldots,s_i\}$.
To generate more, we draw a block of pivots $\set{T} = \{t_1,\ldots,t_{k}\}$ iid from the leverage score distribution \cref{eq:lev} and form
\begin{equation*}
	\mat{C} \coloneqq (\Id - \mat{\Pi}_i) \mat{Q}^*(:,\set{T}).
\end{equation*}
A total of $k$ steps of the rejection sampling loop can now be implemented using only information in the matrix $\mat{C}^*\mat{C}$ and the leverage scores $\ell_{t_1},\ldots,\ell_{t_{k}}$ using the \textsf{RejectionSampleSubmatrix} algorithm from \cite{ETW24a}; see \cref{alg:rejectionsamplesubmatrix}.
This procedure will accept a subset $\set{T}'\subseteq\set{T}$ of the proposed pivots, which are appended $\set{S} \gets \set{S} \cup \set{T}'$ to $\set{S}$.
We repeat these block rejection sampling steps until $\set{S}$ has size $k$.

\begin{algorithm}[t]
	\caption{\textsf{RejectionSampleSubmatrix} \cite[Alg.~2.1]{ETW24a}} \label{alg:rejectionsamplesubmatrix}
	\begin{algorithmic}[1]
		\Require Psd matrix $\mat{H} \in \mathbb{C}^{k\times k}$, leverage scores $\vec{\ell} \in \real_+^k$, proposals $\set{T} = \{ t_1,\ldots,t_k \}$
		\Ensure Accepted proposals $\set{T}' \subseteq \set{T}$
		\State $\set{T}' \gets \emptyset$ 
		\For{$i = 1,\ldots,b$}
		\If{$\vec{\ell}(i) \cdot \Call{Rand}{\,} < \mat{H}(i,i)$} \Comment{Accept or reject}
		\State $\set{T}' \gets \set{T}' \cup \{t_i\}$ \Comment{If accept, induct pivot}
		\State $\mat{H}(i:b,i:b) \gets \mat{H}(i:b,i:b) - \mat{H}(i:b,i)\mat{H}(i,i:b)/\mat{H}(i,i)$ \Comment{Eliminate}
		\EndIf
		\EndFor
	\end{algorithmic}
\end{algorithm}

Second, we maintain a \QR decomposition of $\mat{Q}^*(:,\set{S})$ throughout the course of the algorithm where the orthogonal factor is maintained as a product of Householder reflectors \cite[App.~A]{Epp25a}.
As such, our approach maintains the same numerical stability properties as the sequential, unblocked algorithm of Cortinovis and Kressner \cite{CK24}.

\begin{algorithm}[t]
	\caption{\RejectionRPQR \cite{BTA23,DCMW19}: Efficient block implementation}
	\label{alg:rejectionrpqr}
	\begin{algorithmic}[1]
		\Require Matrix $\mat{Q} \in \complex^{m \times k}$ with orthonormal columns
		\Ensure Subset $\set{S} \subseteq \{1,\ldots,m\} \sim \VS{k}(\mat{Q})$, matrices $\mat{U},\mat{R} \in \complex^{k\times k}$ defining \QR decomposition $\mat{Q}^*(:,\set{S}) = \mat{U}\mat{R}$
		\State $\set{S} \gets \emptyset, (\mat{U},\mat{R}) \gets \Call{EmptyQRObject}{\,}$ \Comment{Householder \QR object}
		\State $\vec{\ell} \gets \Call{SquaredRowNorms}{\mat{Q}}$ \Comment{Compute leverage scores}
		\While{$|\set{S}| < k$}
		\State Draw $\set{T} = \{t_1,\ldots,t_k\}$ iid with $\prob\{t_i = j\} = \ell_j / k$
		\State $\mat{C} \gets (\Id - \mat{U}\mat{U}^*)\mat{Q}^*(:,\set{T})$, $\mat{H} \gets \mat{C}^*\mat{C}$
		\State $\set{T}' \gets \Call{RejectionSampleSubmatrix}{\mat{H},\vec{\ell}(\set{T}),\set{T}}$ \Comment{\cref{alg:rejectionsamplesubmatrix}}
		\State \textbf{if} $|\set{T}'| > k - |\set{S}|$ \textbf{then} $\set{T}' \gets (\text{first $k-|\set{S}|$ elements of $\set{T}'$})$
		\State $(\mat{U},\mat{R}) \gets \Call{QRUpdate}{(\mat{U},\mat{R}),\mat{Q}^*(:,\set{T}')}$ \Comment{Update \QR \cite[App.~A]{Epp25a}}
		\State $\set{S} \gets \set{S} \cup \set{T}'$
		\EndWhile
	\end{algorithmic}
\end{algorithm}

Combining these two approaches, we obtain a fast, numerically stable block implementation of \RejectionRPQR; see \cref{alg:rejectionrpqr}.
It has the same $\order(mk + k^3\log k)$ expected runtime as \cite{BTA23,DCMW19}, and it is much faster in practice.

\section{Sketchy adaptive randomized pivoting} \label{sec:skarp}

As we will see next section, the approximation $\Ahat_2 = \mat{A}\mat{A}(\set{S},:)^\dagger\cdot \mat{A}(\set{S},:)$ produced by \ARP is often much more accurate than the approximation $\Ahat_1 = \mat{Q}\mat{Q}(\set{S},:)^{-1} \cdot \mat{A}(\set{S},:)$.
However, forming $\Ahat_2$ requires evaluating the product $\mat{A}\mat{A}(\set{S},:)^\dagger$ at a cost of $\order(kmn)$ operations, which dominates the roughly $\order(mn + mk^2)$ runtime of the entire \ARP algorithm for computing $\Ahat_1$.

The \emph{oversampled sketchy interpolative decomposition} (OSID) approach of Dong, Chen, Martinsson, and Pearce \cite{DCMP23} yields a fast, approximate way of computing $\Ahat_2$.
Begin by drawing a random embedding $\mat{\Phi}\in\complex^{n\times ck}$, where $c$ is an oversampling factor (e.g., $c=2$).
Using this embedding, we construct the OSID approximation
\begin{equation} \label{eq:osid}
	\Ahat_{\rm OSID} = \mat{W}_{\rm OSID}\cdot \mat{A}(\set{S},:) \quad \text{with }\mat{W}_{\rm OSID} \coloneqq (\mat{A}\mat{\Phi})(\mat{A}(\set{S},:)\mat{\Phi})^\dagger.
\end{equation}
The pseudoinverse can be applied stably via a \QR decomposition of $(\mat{A}(\set{S},:)\mat{\Phi})^*$.

We refer to \ARP with OSID as sketchy adaptive randomized pivoting (\SkARP).
In practice based on the empirical testing from \cite{TYUC19,DM23}, we recommend implementing \SkARP with $\mat{\Phi}$ chosen to be a $n\times 2k$ sparse random embedding with $\zeta = 4$ nonzeros per column.
With this choice, the runtime of \SkARP is $\order(mn + k^2(m+n))$ operations, comparable to the original \ARP algorithm with output $\Ahat_1$.

\section{End-to-end guarantees} \label{sec:end-to-end}
To obtain end-to-end guarantees for \ARP algorithms with sketching, we need to combine the \ARP result (\cref{thm:arp}) with analysis of the range-finder \cref{eq:randomized-rangefinder} and OSID \cref{eq:osid}.
Our results will treat the case where $\mat{\Omega},\mat{\Phi}$ are sparse random embeddings which, up to scaling, take discrete $-1,0,+1$ values.
Extensions to other types of embeddings \cite[secs.~8--9]{MT20} is straightforward.
We use the following sparse embedding construction:
\begin{definition}[SparseStack] \label{def:sparsestack}
	Fix \emph{embedding dimension} $k$ and \emph{row sparsity} $\zeta$, and assume $b = k / \zeta$ is an integer.
	A \emph{SparseStack embedding} is the random matrix
	\begin{equation*}
		\mat{\Omega} = \frac{1}{\zeta^{1/2}}\begin{bmatrix}
			\varrho_{11} \evec_{s_{11}}^* & \cdots & \varrho_{1\zeta} \evec_{s_{1\zeta}}^* \\
			\varrho_{21} \evec_{s_{21}}^* & \cdots & \varrho_{2\zeta} \evec_{s_{2\zeta}}^* \\
			\vdots & \ddots & \vdots \\
			\varrho_{n1} \evec_{s_{n1}}^* & \cdots & \varrho_{n\zeta} \evec_{s_{n\zeta}}^*
		\end{bmatrix} \in \real^{n\times k} \quad \text{where } \begin{cases}
			\varrho_{ij} \stackrel{\rm iid}{\sim} \textsc{Unif} \{\pm 1\}, \\
			s_{ij} \stackrel{\rm iid}{\sim} \textsc{Unif} \{1,\ldots,b\}.
		\end{cases}
	\end{equation*}
\end{definition}

Analysis of randomized linear algebra primitives with sparse embeddings has been an active area of research over the past decade \cite{CDD25,Coh16,NN13,Tro25}.
For our purposes, the best results appear in \cite{CEMT25}.
We use the following simplified version of this paper's results:

\begin{theorem}[Linear algebra with SparseStacks] \label{thm:sparsestacks}
	Let $\mat{A} \in \complex^{m\times n}$ be a matrix and $\mat{\Omega} \in \real^{n\times k}$ be a SparseStack with row-sparsity $\zeta$.
	There exists universal constants $\mathrm{c}_1,\mathrm{c}_2,\mathrm{c}_3 > 0$ such that the following properties hold with 90\% probability:
	\begin{enumerate}
		\item \textbf{Range-finder.}
		Fix $r > 0$, and consider the range-finder $\mat{Q} = \orth(\mat{A}\mat{\Omega})$.
		If $k \ge \mathrm{c}_1r$ and $\zeta \ge \mathrm{c}_2\log r$, then $\norm{\mat{A} - \mat{Q}\mat{Q}^*\mat{A}}_{\rm F} \le \mathrm{c}_3 \norm{\mat{A} - \lowrank{\mat{A}}_r}_{\rm F}$.
		\item \textbf{Sketched pseudoinverse.}
		Let $\mat{B} \in \complex^{\ell\times n}$, and assume $n \ge \ell$.
		If $k \ge \mathrm{c}_1\ell$ and $\zeta \ge \mathrm{c}_2\log \ell$, then $\norm{\smash{(\mat{A}\mat{\Omega})(\mat{B}\mat{\Omega})^\dagger}\mat{B} - \mat{A}}_{\rm F} \le \mathrm{c}_3 \norm{\smash{\mat{A}\mat{B}^\dagger}\mat{B} - \mat{A}}_{\rm F}$.
	\end{enumerate}
\end{theorem}

Combining this result with \cref{thm:arp} immediately yields bounds for \ARP implementations using sparse sketching.
\begin{theorem}[Adaptive randomized pivoting: End-to-end guarantees] \label{thm:end-to-end}
	Let $\mat{A} \in \complex^{m\times n}$ be a matrix, and fix a rank $r > 0$.
	Choose $k\ge \mathrm{c}_1 r$, $p \ge \mathrm{c}_1 k$, and $\zeta \ge \mathrm{c}_2 \log k$, and form SparseStacks $\mat{\Omega} \in \real^{n\times k}$ and $\mat{\Phi} \in \real^{n\times p}$ with row sparsity $\zeta$.
	The \ARP and \SkARP algorithms with embeddings $\mat{\Omega}$ and $\mat{\Phi}$ produce rank-$k$ approximations $\Ahat$ satisfying
	\begin{equation*}
		\operatorname{Median}(\norm{\mat{A} - \Ahat}_{\rm F}) \le \mathrm{c}r^{1/2} \norm{\mat{A} - \lowrank{\mat{A}}_r}_{\rm F} \quad \text{for a universal constant } \mathrm{c} > 0
	\end{equation*}
\end{theorem}

\begin{proof}
	Throughout this proof, we let $\mathrm{c} > 0$ denote an arbitrary universal constant whose value we permit to change on every usage, even on the same line.
	By \cref{thm:sparsestacks}, the following range-finder guarantee holds with at least 90\% probability:
	\begin{equation*}
		\norm{\mat{A} - \mat{Q}\mat{Q}^*\mat{A}}_{\rm F} \le \mathrm{c} \norm{\mat{A} - \lowrank{\mat{A}}_r}_{\rm F}. 
	\end{equation*}
	In the event this bound holds, \cref{thm:arp} implies
	\begin{equation*}
		\expect_{\set{S}} \norm{\mat{A} - \Ahat_1}_{\rm F}^2 \le (r+1)\norm{\mat{A} - \mat{Q}\mat{Q}^*\mat{A}}_{\rm F}^2 \le \mathrm{c}r \norm{\mat{A} - \lowrank{\mat{A}}_r}_{\rm F}^2.
	\end{equation*}
	Here, $\expect_{\set{S}}$ denotes the expectation with respect to the randomness in the set $\set{S}$.
	By Markov's inequality and the comparison \cref{eq:comparision-A1-A2}, we conclude that
	\begin{equation*}
		\norm{\mat{A} - \Ahat_2}_{\rm F} \le \norm{\mat{A} - \Ahat_1}_{\rm F} \le \mathrm{c}r^{1/2} \norm{\mat{A} - \lowrank{\mat{A}}_r}_{\rm F}
	\end{equation*}
	with probability at least 80\%.
	The stated bound for the median error of \ARP follows.
	To analyze \SkARP, we apply \cref{thm:sparsestacks} again to conclude
	\begin{align*}
		\norm{\mat{A} - \Ahat_{\rm OSID}}_{\rm F} &= \norm{\mat{A} - (\mat{A}\mat{\Phi})(\mat{A}(\set{S},:)\mat{\Phi})^\dagger \mat{A}(\set{S},:)}_{\rm F} \\
		&\le \mathrm{c} \norm{\mat{A} - \mat{A}\mat{A}(\set{S},:)^\dagger \mat{A}(\set{S},:)}_{\rm F}= \mathrm{c}\norm{\mat{A} - \Ahat_2}_{\rm F} \le  \mathrm{c}r^{1/2} \norm{\mat{A} - \lowrank{\mat{A}}_r}_{\rm F}
	\end{align*}
	with probability at least 70\%.
	The stated median bound follows.
\end{proof}

Let me highlight two limitations of this analysis.
First, the constants in \cref{thm:end-to-end} are unspecified and will not be small using this proof technique.
Second, \cref{thm:end-to-end} only bounds the \emph{median} error.
The source of both limitations is the usage of Markov's inequality in the proof of \cref{thm:end-to-end} and the imported result \cref{thm:sparsestacks}.
I believe that, with better analysis, it should be possible to obtain bounds on the error with small explicit constants and that hold with high probability.
Developing such bounds would likely be challenging due to the need to obtain sharp analysis for the SparseStack or another sparse random embedding (a long-standing open question \cite[sec.~5.2]{ABB+26}) and concentration bounds for the volume sampling least-squares residual.

\section{Experiments} \label{sec:experiments}

In this section, we provide experimental results to evaluate the speed and accuracy of several versions of the \ARP procedure.
Code may be found at \url{https://github.com/eepperly/Adaptive-Randomized-Pivoting}.

\subsection{Experimental setup}

We compare three versions of \ARP, each using a different choice of the interpolation matrix $\mat{W}$.
We call these methods \ARP (which outputs $\Ahat_1$), \ProjARP (which outputs $\Ahat_2$), and \SkARP (using OSID, \cref{sec:skarp}).
As baselines, we test the sketchy pivoted \QR method with OSID (\textsf{SkQR}, \cite{VM17,DM23,DCMP23}) and the accelerated randomly pivoted \QR method (\textsf{RPQR}, \cite{ETW24a}). 
A comparison of methods is provided in \cref{tab:method-comparison}.

All methods are implemented in MATLAB.
All randomized embeddings, both for range-finding \cref{eq:randomized-rangefinder} and OSID \cref{eq:osid}, are sparse sign embeddings with row-sparsity $\zeta = 4$.
Following \cite{DCMP23}, use oversampling factor $c=2$ for OSID.
All experiments use real values and store numbers in double precision.

\begin{table}[t]
	\centering
	\caption{Summary of row interpolative decomposition methods. Runtimes consist of the end-to-end cost of forming the matrix $\mat{Q}$, identifying the index set $\set{S}$, and building the interpolation matrix $\mat{W}$.}
	\begin{tabular}{lll}  \toprule
		Method  & Matrix $\mat{W}$  & Runtime \\\midrule
		\ARP & $\mat{Q}\mat{Q}(\set{S},:)^{-1}$ & $\order(mn + k^2m)$ \\
		\ProjARP & $\mat{A}\mat{A}(\set{S},:)^\dagger$ & $\order(kmn)$ \\
		\SkARP & $(\mat{A}\mat{\Phi})(\mat{A}(\set{S},:)\mat{\Phi})^\dagger$ & $\order(mn + k^2(m+n))$ \\\midrule
		\SkQR & $(\mat{A}\mat{\Phi})(\mat{A}(\set{S},:)\mat{\Phi})^\dagger$ & $\order(mn + k^2(m+n))$ \\
		\RPQR & $\mat{A}\mat{A}(\set{S},:)^\dagger$ & $\order(kmn)$ \\ \bottomrule
	\end{tabular}
	\label{tab:method-comparison}
\end{table}

\subsection{Runtime experiments}
To evaluate the speed of \ARP, we test on two examples, one dense and one sparse.
These examples are generated as
\begin{equation*}
	\mat{A}_j = \diag(i^{-2} : i=1,\ldots,m) \cdot \mat{G}_j
\end{equation*}
where $\mat{G}_1 \in \real^{10^4\times 10^4}$ is a dense Gaussian matrix and $\mat{G}_2 \in \real^{10^6\times 10^4}$ is a sparse Gaussian matrix with 30 nonzero entries per column in random positions.
We collect runtimes using MATLAB's \texttt{timeit}, which reports a median of multiple trials.
We implement both \SkARP and \ARP with the fast blocked implementation of \RejectionRPQR.
As a point of comparison, we also implement \ARP according to the pseudocode from Cortinovis and Kressner's pseudocode, which uses un-blocked sequential sampling and applies a Householder reflector to the entire matrix at each step.

\begin{figure}[t]
	\centering
	\includegraphics[width=0.48\linewidth]{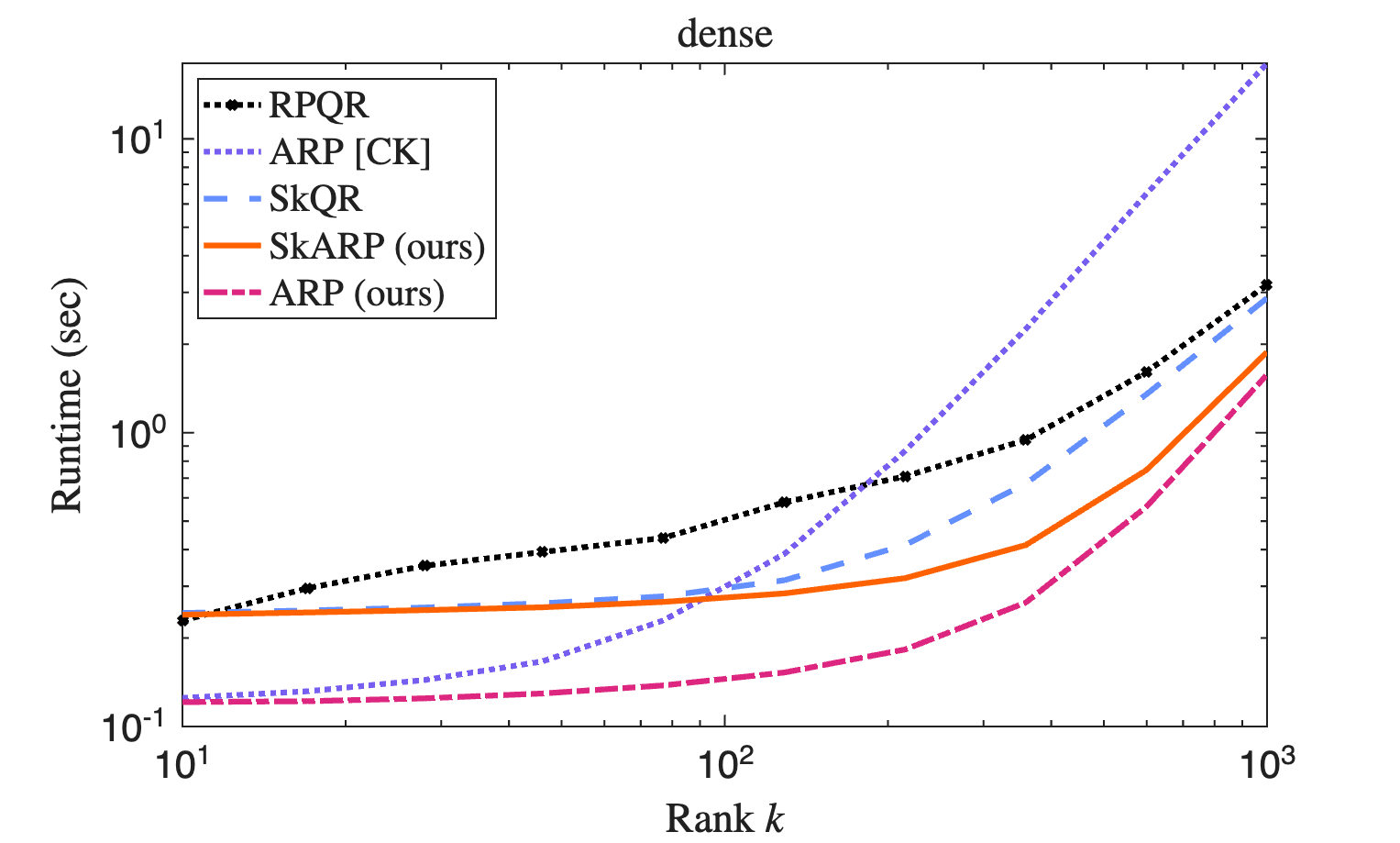}
	\includegraphics[width=0.48\linewidth]{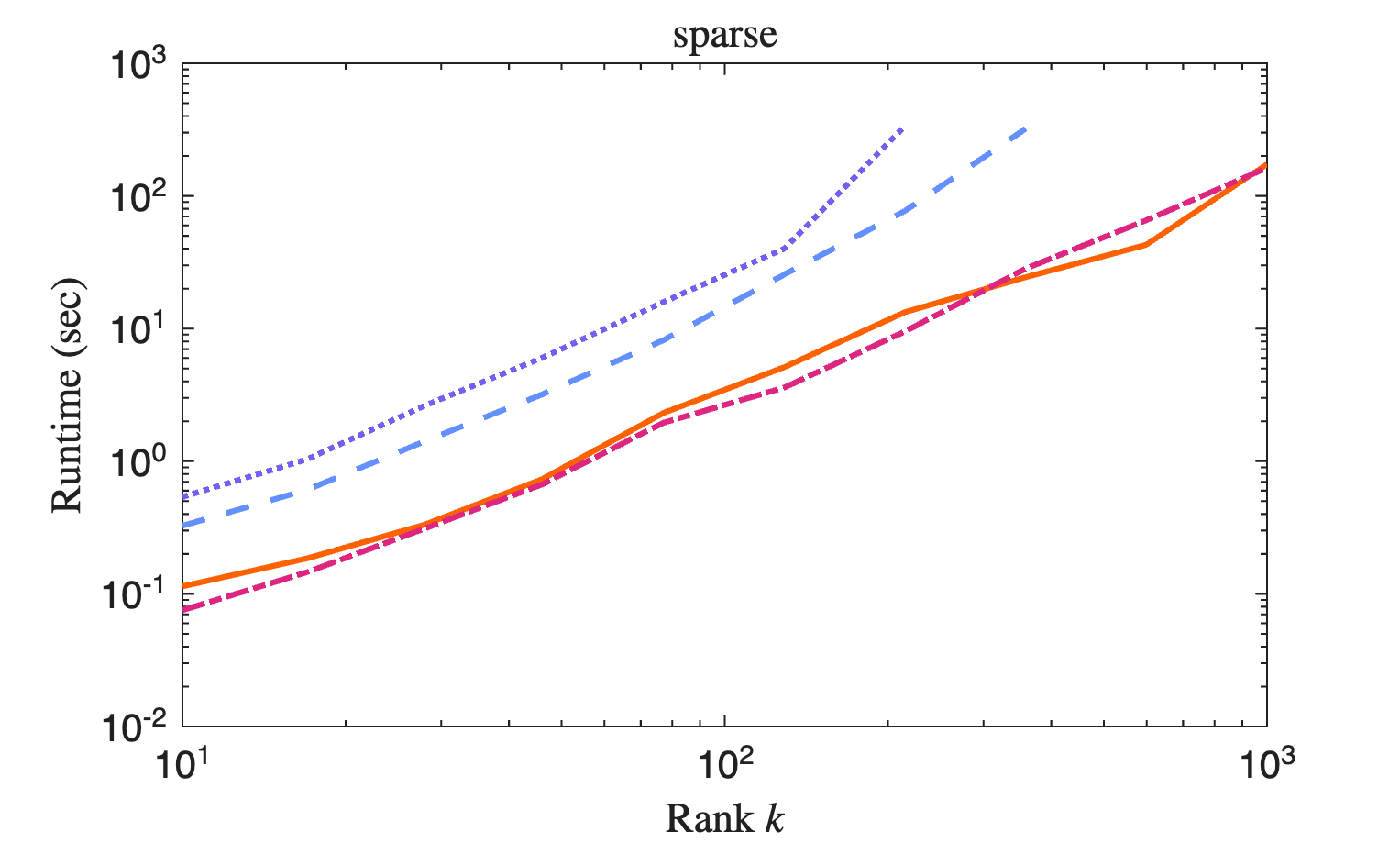}
	\caption{\textbf{\ARP speed tests.}
		Runtime for five row interpolative decomposition methods on dense (\emph{left}) and sparse (\emph{right}) test matrices, described in text.}
	\label{fig:time}
\end{figure}

Results are shown in \cref{fig:time}.
For the \textit{dense} example, this paper's \ARP implementations achieve a maximum speedup of \textbf{2$\times$ over randomly pivoted \QR, 1.8$\times$ speedup over sketchy pivoted \QR, and 11$\times$ over an \ARP implementation using Cortinovis and Kressner's pseudocode}.
For the \textit{sparse} example, this paper's \ARP implementations achieve maximal speedups of 4$\times$ over sketchy pivoted \QR and 35$\times$ over an \ARP implementation using Cortinovis and Kressner's pseudocode.
Further speedups for \ARP are possible by implementing the \textsf{RejectionSampleSubmatrix} subroutine in a low-level programming language like \textsf{C}\texttt{++}.
These experiments suggest that this paper's versions of \ARP are among the fastest algorithms for computing an interpolative decomposition.

\subsection{Accuracy experiments}
To evaluate the accuracy of \ARP variants, we test on two matrices.
Our first example, \emph{kernel}, tabulates the inverse distances of a set of points $\vec{x}_i,\vec{y}_j \in \real^2$:
\begin{equation*}
	\mat{A}(i,j) = \frac{1}{\norm{\smash{\vec{x}_i - \vec{y}_j}}} \quad \text{for } i,j = 1,\ldots,n.
\end{equation*}
We set $n \coloneqq 10^4$ and choose the points $\{\vec{x}_i\}$ and $\{\vec{y}_j\}$ to be equispaced Cartesian grids on $[0,1)^2$ and $[1,2)\times[0,1)$, respectively.
This example is similar to matrices that occur in rank-structured matrix applications \cite{Mar11,Wil21}.
Our second matrix, \emph{genetics}, comes from the GSE10072 cancer genetics data set from the National Institutes of Health and is an established benchmark for subset selection \cite{SE16,CK24}.
This matrix may be downloaded at \url{https://ftp.ncbi.nlm.nih.gov/geo/series/GSE10nnn/GSE10072/matrix/}.
We perform row selection on its transpose, which has dimensions $107\times22283$.
We run for 100 trials for each example, and we report the relative Frobenius-norm error.
Lines show the mean error, and shaded regions track the maximum and minimum error.

\begin{figure}[t]
	\centering
	\includegraphics[width=0.48\linewidth]{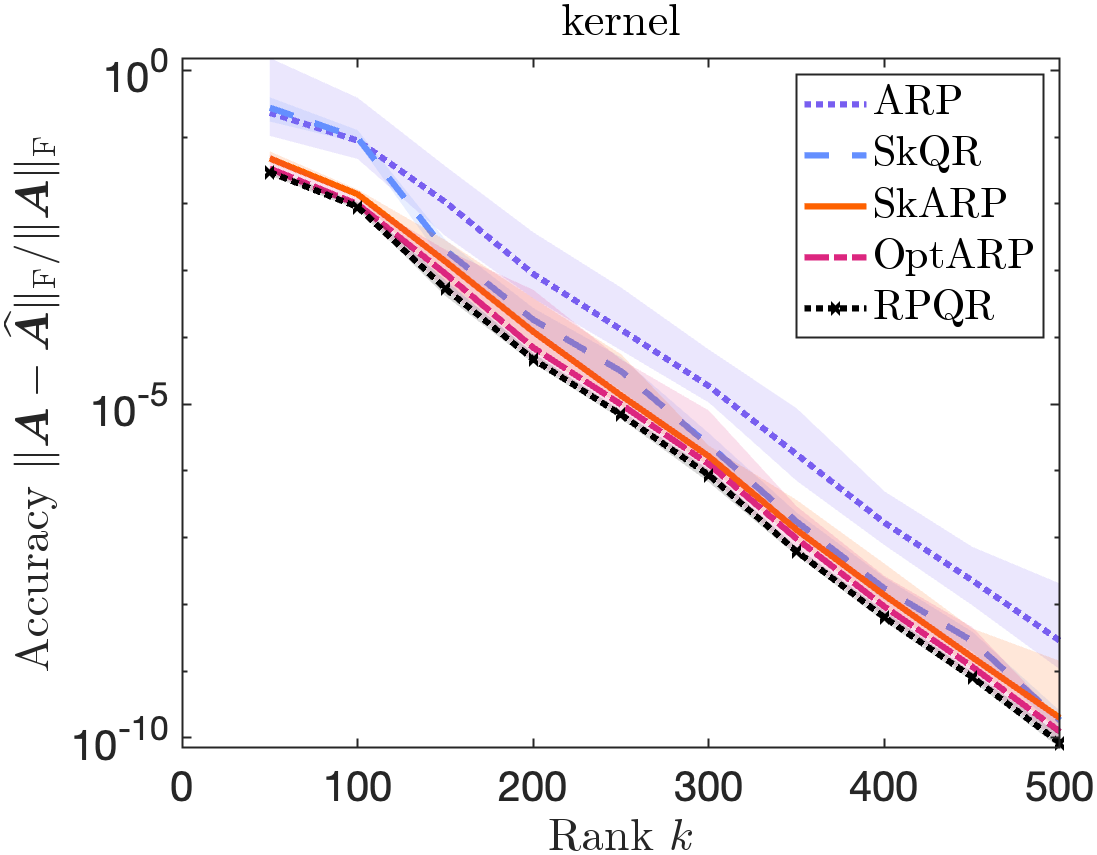}
	\includegraphics[width=0.48\linewidth]{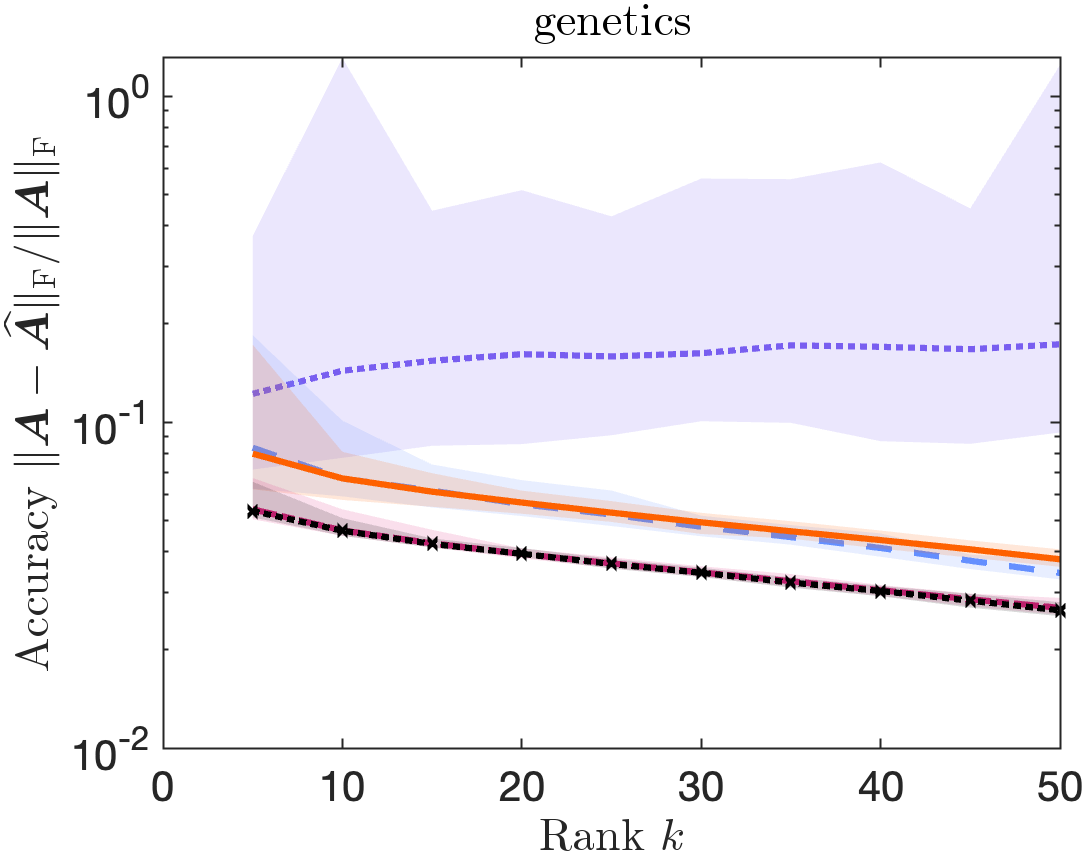}
	\caption{\textbf{\ARP accuracy tests.}
		Relative error for five row interpolative decomposition methods on kernel (\emph{left}) and genetics (\emph{right}) test matrices.
		Lines show mean of 100 trials, and shaded regions show the maximum and minimum errors.}
	\label{fig:accuracy}
\end{figure}

\Cref{fig:accuracy} shows the results.
On both problems, randomly pivoted \QR is the most accurate, with the ``\textsf{OptARP}'' approximation $\Ahat_2$ achieving nearly the same accuracy.
Sketching pivoted and sketchy \ARP (\textsf{SkARP}) are the next most accurate methods, achieving errors roughly 1.5$\times$ higher than randomly pivoted \QR and \textsf{OptARP} while being meaningfully faster (cf.\ \cref{fig:time}).
The cheapest \textsf{ARP} approximation $\Ahat_1$ is meaningfully less accurate than other methods, which is to be expected due to the factor of $k+1$ in \cref{thm:arp}.
Whether this loss of accuracy is acceptable depends on the application.
For the \textit{kernel} example, the rate of convergence is geometric and one may be willing to tolerate a modestly lower quality approximation.
For the \textit{genetics} example, the rate of convergence for all methods is slow and the error of the plain \ARP approximation $\Ahat_1$ actually \emph{increases} with $k$.
On the basis of these experiments, we recommend the \textsf{SkARP} variant for general-purpose use and the \textsf{OptARP} variant for applications where the highest-possible accuracy is critical.

\section*{Acknowledgments}

I give my warm thanks to Chris Cama\~no, Alice Cortinovis, Micha\l\ Derezi\'nski, Daniel Kressner, Raphael Meyer, Arvind Saibaba, Joel Tropp, and Robert Webber for helpful conversations.

\appendix

\section{Optimality} \label{app:optimaity}
The following result shows that the factor of $k+1$ in \cref{thm:active-regression} cannot be improved.
As such, volume sampling is an optimal method for active linear regression with $\ell = k$ queries.
\begin{proposition}[Optimality of volume sampling] \label{thm:optimality}
	There exists  $\mat{X} \in \real^{(k+1)\times k}$ and $\vec{y} \in \real^{k+1}$ such that for any subset $\set{S} \subseteq \{1,\ldots,k+1\}$ of $k$ elements,
	\begin{equation*}
		\norm{\mat{X}\mat{X}(\set{S},:)^{-1}\vec{y}(\set{S}) - \vec{y}}^2 = (k+1) \norm{\smash{(\Id - \mat{X}\mat{X}^\dagger)\vec{y}}}^2.
	\end{equation*}
\end{proposition}
This result appears without proof in \cite[Prop.~6]{DW17}.
For completeness, we prove it here.

\begin{proof}[Proof of \cref{thm:optimality}]
	Choose $\vec{y} \coloneqq \vec{1}_{k+1} \in \real^{k+1}$ to be the vector of all ones, and choose any full-rank $\mat{X} \in \real^{(k+1)\times k}$ satisfying $\mat{X}^*\vec{y} = \vec{0}$.
	Since $\mat{X}^*\vec{y} = \vec{0}$, the least-squares solution is zero and
	\begin{equation} \label{eq:active-lower-1}
		\min_{\vec{\beta} \in \complex^k} \norm{\mat{X}\vec{\beta} - \vec{y}}^2 = \norm{\vec{y}}^2 = k+1.
	\end{equation}
	
	Now, let $\set{S}$ be any set of $k$ row indices, and let $u$ denote the sole element of $\{1,\ldots,k+1\} \setminus \set{S}$.
	The matrix $\mat{X}(\set{S},:)$ is invertible the left nullspace $\operatorname{null}(\mat{X}^*) = \operatorname{span}\{\vec{1}_{k+1}\}$ does not contain a nonzero vector supported on $\set{S}$. 
	Define $\vec{\hat{\beta}} \coloneqq \mat{X}(\set{S},:)^{-1}\vec{y}(\set{S})$.
	The vector $\vec{\hat{\beta}}$ exactly satisfies the equations in the $\set{S}$ positions: 
	\begin{equation*}
		(\mat{X}\vec{\hat{\beta}})(\set{S}) = \vec{1}_k.
	\end{equation*}
	However, $\mat{X}\vec{\hat{\beta}}$ must be in the range of $\mat{X}$, which is orthogonal to $\vec{1}_{k+1}$.
	Ergo, for the single index $u \notin \set{S}$, we must have
	\begin{equation*}
		(\mat{X}\vec{\hat{\beta}})_u = -\sum_{s\in\set{S}} (\mat{X}\vec{\hat{\beta}})_s = -k.
	\end{equation*}
	We conclude that
	\begin{equation} \label{eq:active-lower-2}
		\norm{\smash{\mat{X}\vec{\hat{\beta}} - \vec{y}}}^2 = \sum_{s \in \set{S}}\underbrace{[(\mat{X}\vec{\hat{\beta}})_s - 1]^2}_{=0} + \underbrace{[(\mat{X}\vec{\hat{\beta}})_u - 1]^2}_{=(k+1)^2} = (k+1)^2.
	\end{equation}
    Combining \cref{eq:active-lower-1,eq:active-lower-2} yields the stated result.
\end{proof}

\bibliographystyle{siamplain}
\bibliography{refs}

@article{ABB+26,
  title = {Linear Systems and Eigenvalue Problems: Open Questions from a {{Simons}} Workshop},
  author = {Amsel, Noah and others},
  year = 2026,
  journal = {arXiv preprint \href{https://arxiv.org/abs/2602.05394v2}{arXiv:2602.05394v2}},
}

@article{MD09a,
  author =        {Mahoney, Michael W. and Drineas, Petros},
  journal =       {Proceedings of the National Academy of Sciences},
  month =         jan,
  number =        {3},
  pages =         {697--702},
  title =         {{{CUR}} Matrix Decompositions for Improved Data
                   Analysis},
  volume =        {106},
  year =          {2009},
doi =           {10.1073/pnas.0803205106},
}

@inproceedings{BMD08,
author =        {Boutsidis, Christos and Mahoney, Michael W. and
                   Drineas, Petros},
  booktitle =     {Proceedings of the 14th {{ACM SIGKDD}} International
                   Conference on {{Knowledge}} Discovery and Data
                   Mining},
  month =         aug,
  pages =         {61--69},
  publisher =     {ACM},
  title =         {Unsupervised Feature Selection for Principal
                   Components Analysis},
  year =          {2008},
doi =           {10.1145/1401890.1401903},
  isbn =          {978-1-60558-193-4},
}

@inproceedings{DPPL24,
  author =        {Dong, Yijun and Pan, Xiang and Phan, Hoang and
                   Lei, Qi},
  booktitle =     {Workshop on Machine Learning and Compression,
                   {{NeurIPS}} 2024},
  title =         {Randomly Pivoted {V}-optimal Design: Fast Data
                   Selection under Low Intrinsic Dimension},
  year =          {2024},
  url = {https://openreview.net/pdf?id=WPvQVQrbch}
}

@inproceedings{DCMW19,
  author =        {Derezi{\'n}ski, Micha{\l} and Clarkson, Kenneth L. and
                   Mahoney, Michael W. and Warmuth, Manfred K.},
  booktitle =     {Conference on {{Learning Theory}}},
  month =         jun,
  pages =         {1050--1069},
  publisher =     {PMLR},
  title =         {Minimax Experimental Design: {{Bridging}} the Gap
                   between Statistical and Worst-Case Approaches to
                   Least Squares Regression},
  url = {https://proceedings.mlr.press/v99/derezinski19b.html},
  year =          {2019},
}

@phdthesis{Epp25a,
  title = {Make the Most of What You Have: {{Resource-efficient}} Randomized Algorithms for Matrix Computations},
  author = {Epperly, Ethan N.},
  year = {2025},
  school = {California Institute of Technology}
}

@article{Mar11,
  author =        {Martinsson, P. G.},
  journal =       {SIAM Journal on Matrix Analysis and Applications},
  month =         oct,
  number =        {4},
  pages =         {1251--1274},
  title =         {A Fast Randomized Algorithm for Computing a
                   Hierarchically Semiseparable Representation of a
                   Matrix},
  volume =        {32},
  year =          {2011},
  doi =           {10.1137/100786617},
}

@phdthesis{Wil21,
  author =        {Wilber, Heather Denise},
  title =         {Computing Numerically With Rational Functions},
  school = {Cornell University},
  year =          {2021},
  url = {https://heatherw3521.github.io/phd_thesis.pdf},
}

@article{OT10,
  author =        {Oseledets, Ivan and Tyrtyshnikov, Eugene},
  journal =       {Linear Algebra and its Applications},
  number =        {1},
  pages =         {70--88},
  publisher =     {Elsevier},
  title =         {{{TT}}-Cross Approximation for Multidimensional
                   Arrays},
  volume =        {432},
  year =          {2010},
  doi = {10.1016/j.laa.2009.07.024}
}

@article{TSL24b,
  author =        {Tindall, Joseph and Stoudenmire, Miles and
                   Levy, Ryan},
  month =         oct,
  journal =        {arXiv preprint \href{https://arxiv.org/abs/2410.03572v1}{arXiv:2410.03572v1}},
  title =         {Compressing Multivariate Functions with Tree Tensor
                   Networks},
  year =          {2024},
}

@book{GV13,
author =        {Golub, Gene H. and Van Loan, Charles F.},
  edition =       {Fourth},
  pages =         {xxi + 756},
  publisher =     {Johns Hopkins Press},
  series =        {Johns Hopkins Studies in the Mathematical Sciences},
  title =         {Matrix Computations},
  year =          {2013},
bibsource =
  {http://www.math.utah.edu/pub/bibnet/authors/g/golub-gene-h.bib;
  http://www.math.utah.edu/pub/bibnet/authors/l/lanczos-cornelius.bib;
  http://www.math.utah.edu/pub/tex/bib/numana2010.bib},
  isbn =          {1-4214-0794-9 (hardcover), 1-4214-0859-7 (e-book)},
}

@article{GE94,
  author =        {Gu, Ming and Eisenstat, Stanley C.},
  journal =       {SIAM Journal on Matrix Analysis and Applications},
  month =         oct,
  number =        {4},
  pages =         {1266--1276},
  publisher =     {SIAM},
  title =         {A {{Stable}} and {{Efficient Algorithm}} for the
                   {{Rank-One Modification}} of the {{Symmetric
                   Eigenproblem}}},
  volume =        {15},
  year =          {1994},
doi =           {10.1137/S089547989223924X},
}

@inproceedings{FKV98,
  author =        {Frieze, Alan and Kannan, Ravi and Vempala, Santosh},
  booktitle =     {Proceedings 39th {{Annual Symposium}} on
                   {{Foundations}} of {{Computer Science}} ({{Cat}}.
                   {{No}}.{{98CB36280}})},
  month =         nov,
  pages =         {370--370},
  publisher =     {IEEE},
  title =         {Fast {Monte-Carlo} Algorithms for Finding Low-Rank
                   Approximations},
  year =          {1998},
doi =           {10.1109/SFCS.1998.743487},
  isbn =          {978-0-8186-9172-0},
}

@article{Woo14,
  title={Sketching as a Tool for Numerical Linear Algebra},
  author={Woodruff, David P.},
  journal={Foundations and Trends in Theoretical Computer Science},
  volume={10},
  number={1--2},
  pages={1--157},
  year={2014},
  publisher={Now Publishers},
doi =           {10.1561/0400000060},
}

@inproceedings{BMD09,
  author =        {Boutsidis, Christos and Mahoney, Michael W. and
                   Drineas, Petros},
  booktitle =     {Proceedings of the {{Twentieth Annual ACM-SIAM
                   Symposium}} on {{Discrete Algorithms}}},
  month =         jan,
  pages =         {968--977},
  publisher =     {SIAM},
  title =         {An Improved Approximation Algorithm for the Column
                   Subset Selection Problem},
  year =          {2009},
  doi =           {10.1137/1.9781611973068.105},
  isbn =          {978-0-89871-680-1 978-1-61197-306-8},
}

@inproceedings{DRVW06,
author =        {Deshpande, Amit and Rademacher, Luis and
                   Vempala, Santosh and Wang, Grant},
  booktitle =     {Proceedings of the Seventeenth Annual {{ACM-SIAM}}
                   Symposium on {{Discrete}} Algorithms},
  pages =         {1117--1126},
  publisher =     {ACM},
  title =         {Matrix Approximation and Projective Clustering via
                   Volume Sampling},
  year =          {2006},
  doi =           {10.1145/1109557.1109681},
  isbn =          {978-0-89871-605-4},
}

@inproceedings{DV06,
author =        {Deshpande, Amit and Vempala, Santosh},
  booktitle =     {Approximation, Randomization, and Combinatorial
                   Optimization. Algorithms and Techniques},
  pages =         {292--303},
  publisher =     {Springer},
  series =        {Lecture {{Notes}} in {{Computer Science}}},
  title =         {Adaptive Sampling and Fast Low-Rank Matrix
                   Approximation},
  year =          {2006},
  doi =           {10.1007/11830924_28},
  isbn =          {978-3-540-38045-0},
}

@article{CETW,
  title = {Randomly Pivoted {Cholesky}: {Practical} Approximation of a Kernel Matrix with Few Entry Evaluations},
  author = {Chen, Yifan and Epperly, Ethan N. and Tropp, Joel A. and Webber, Robert J.},
  year = {2025},
  journal = {Communications on Pure and Applied Mathematics},
  volume = {78},
  number = {5},
  pages = {995--1041},
  issn = {1097-0312},
  doi = {10.1002/cpa.22234},
}

@article{ETW24a,
  title = {Embrace Rejection: Kernel Matrix Approximation by Accelerated Randomly Pivoted {{Cholesky}}},
  shorttitle = {Embrace {{Rejection}}},
  author = {Epperly, Ethan N. and Tropp, Joel A. and Webber, Robert J.},
  year = 2025,
  month = dec,
  journal = {SIAM Journal on Matrix Analysis and Applications},
  pages = {2527--2557},
  issn = {0895-4798},
  doi = {10.1137/24M1699048},
  urldate = {2025-12-17}
}

@article{DW17,
  author =        {Derezi{\'n}ski, Micha{\l} and Warmuth, Manfred K.},
  journal =       {Advances in Neural Information Processing Systems},
  title =         {Unbiased Estimates for Linear Regression via Volume
                   Sampling},
  volume =        {30},
  url = {https://dl.acm.org/doi/10.5555/3294996.3295068},
  year =          {2017},
}

@article{DW18a,
  author =        {Derezi{\'n}ski, Micha{\l} and Warmuth, Manfred K.},
  journal =       {Journal of Machine Learning Research},
  number =        {23},
  pages =         {1--39},
  title =         {Reverse Iterative Volume Sampling for Linear
                   Regression},
  url = {https://dl.acm.org/doi/10.5555/3291125.3291148},
  volume =        {19},
  year =          {2018},
}

@article{VM17,
  author =        {Voronin, Sergey and Martinsson, Per-Gunnar},
  journal =       {Advances in Computational Mathematics},
  month =         jun,
  number =        {3},
  pages =         {495--516},
  title =         {Efficient Algorithms for {{CUR}} and Interpolative
                   Matrix Decompositions},
  volume =        {43},
  year =          {2017},
doi =           {10.1007/s10444-016-9494-8},
}

@article{DCMP23,
  title = {Robust Blockwise Random Pivoting: {{Fast}} and Accurate Adaptive Interpolative Decomposition},
  author = {Dong, Yijun and Chen, Chao and Martinsson, Per-Gunnar and Pearce, Katherine},
  year = {2025},
  month = sep,
  journal = {SIAM Journal on Matrix Analysis and Applications},
  pages = {1791--1815},
  publisher = {SIAM},
  doi = {10.1137/24M1678027},
}

@article{CK24,
  author =        {Cortinovis, Alice and Kressner, Daniel},
  month =         dec,
  journal =        {SIAM Journal on Matrix Analysis and Applications, to appear},
  title =         {Adaptive Randomized Pivoting for Column Subset
                   Selection, {{DEIM}}, and Low-Rank Approximation},
  year =          {2025},
  note = {(preprint \href{https://arxiv.org/abs/2412.13992v2}{arXiv:2412.13992})}
}

@article{HMT11,
  author =        {Halko, Nathan and Martinsson, Per-Gunnar and
                   Tropp, Joel A.},
  journal =       {SIAM Review},
  month =         jan,
  number =        {2},
  pages =         {217--288},
  publisher =     {SIAM},
  title =         {Finding Structure with Randomness: {{Probabilistic}}
                   Algorithms for Constructing Approximate Matrix
                   Decompositions},
  volume =        {53},
  year =          {2011},
  doi =           {10.1137/090771806},
}

@article{MT20,
  author =        {Martinsson, Per-Gunnar and Tropp, Joel A.},
  journal =       {Acta Numerica},
  month =         may,
  pages =         {403--572},
  publisher =     {Cambridge University Press},
  title =         {Randomized Numerical Linear Algebra: {{Foundations}}
                   and Algorithms},
  volume =        {29},
  year =          {2020},
  doi =           {10.1017/S0962492920000021},
}

@article{TW23,
  author =        {Tropp, Joel A. and Webber, Robert J.},
  month =         jun,
  journal =        {arXiv preprint \href{https://arxiv.org/abs/2306.12418v3}{arXiv:2306.12418v3}},
  title =         {Randomized Algorithms for Low-Rank Matrix
                   Approximation: {{Design}}, Analysis, and
                   Applications},
  year =          {2023},
}

@article{Gu15,
  author =        {Gu, Ming},
  journal =       {SIAM Journal on Scientific Computing},
  month =         jan,
  number =        {3},
  pages =         {A1139-A1173},
  title =         {Subspace Iteration Randomization and Singular
                   Value Problems},
  volume =        {37},
  year =          {2015},
doi =           {10.1137/130938700},
}

@article{DM21a,
  author =        {Derezi{\'n}ski, Micha{\l} and Mahoney, Michael W.},
  journal =       {Notices of the American Mathematical Society},
  month =         jan,
  number =        {01},
  pages =         {1},
  title =         {Determinantal Point Processes in Randomized
                   Numerical Linear Algebra},
  volume =        {68},
  year =          {2021},
  doi =           {10.1090/noti2202},
}

@article{HKPV06,
  author =        {Hough, J. Ben and Krishnapur, Manjunath and
                   Peres, Yuval and Vir{\'a}g, B{\'a}lint},
  journal =       {Probability Surveys},
  month =         jan,
  number =        {none},
  pages =         {206--229},
  publisher =     {Institute of Mathematical Statistics and Bernoulli
                   Society},
  title =         {Determinantal Processes and Independence},
  volume =        {3},
  year =          {2006},
doi =           {10.1214/154957806000000078},
}

@inproceedings{BTA23,
  author =        {Barthelme, Simon and Tremblay, Nicolas and
                   Amblard, Pierre-Olivier},
  booktitle =     {International {{Conference}} on {{Artificial
                   Intelligence}} and {{Statistics}}},
  pages =         {5582--5592},
  publisher =     {PMLR},
  title =         {A Faster Sampler for Discrete Determinantal
                   Point Processes},
  year =          {2023},
  url = {https://proceedings.mlr.press/v206/barthelme23a.html}
}

@article{TYUC19,
  author =        {Tropp, Joel A. and Yurtsever, Alp and
                   Udell, Madeleine and Cevher, Volkan},
  journal =       {SIAM Journal on Scientific Computing},
  month =         jan,
  number =        {4},
  pages =         {A2430-A2463},
  publisher =     {SIAM},
  title =         {Streaming Low-Rank Matrix Approximation with an
                   Application to Scientific Simulation},
  volume =        {41},
  year =          {2019},
doi =           {10.1137/18M1201068},
}

@article{DM23,
  author =        {Dong, Yijun and Martinsson, Per-Gunnar},
  journal =       {Advances in Computational Mathematics},
  month =         aug,
  number =        {4},
  pages =         {66},
  title =         {Simpler Is Better: A Comparative Study of Randomized
                   Pivoting Algorithms for {{CUR}} and Interpolative
                   Decompositions},
  volume =        {49},
  year =          {2023},
doi =           {10.1007/s10444-023-10061-z},
}

@inproceedings{CDD25,
  title = {Optimal Subspace Embeddings: {{Resolving}} Nelson-Nguyen Conjecture up to Sub-Polylogarithmic Factors},
  booktitle = {Proceedings of the 2026 Annual {{ACM-SIAM}} Symposium on Discrete Algorithms},
  author = {Chenakkod, Shabarish and Derezinski, Michal and Dong, Xiaoyu},
  year = 2026,
  pages = {4501--4559},
  publisher = {SIAM},
  doi = {10.1137/1.9781611978971.163}
}

@inproceedings{Coh16,
  author =        {Cohen, Michael B.},
  booktitle =     {Proceedings of the {{Twenty-Seventh Annual ACM-SIAM
                   Symposium}} on {{Discrete Algorithms}}},
  month =         jan,
  pages =         {278--287},
  publisher =     {SIAM},
  title =         {Nearly Tight Oblivious Subspace Embeddings by Trace
                   Inequalities},
  year =          {2016},
doi =           {10.1137/1.9781611974331.ch21},
  isbn =          {978-1-61197-433-1},
}

@inproceedings{NN13,
author =        {Nelson, Jelani and Nguyen, Huy L.},
  booktitle =     {Proceedings of the Forty-Fifth Annual {{ACM}}
                   Symposium on {{Theory}} of {{Computing}}},
  month =         jun,
  pages =         {101--110},
  publisher =     {ACM},
  title =         {Sparsity Lower Bounds for Dimensionality Reducing
                   Maps},
  year =          {2013},
  doi =           {10.1145/2488608.2488622},
  isbn =          {978-1-4503-2029-0},
}

@article{Tro25,
  author =        {Tropp, Joel A.},
  month =         jan,
  journal =        {arXiv preprint \href{https://arxiv.org/abs/2501.16578v1}{arXiv:2501.16578v1}},
  title =         {Comparison Theorems for the Minimum Eigenvalue of a
                   Random Positive-Semidefinite Matrix},
  year =          {2025},
}

@article{CEMT25,
  title = {Faster Linear Algebra Algorithms with Structured Random Matrices},
  author = {Cama{\~n}o, Chris and Epperly, Ethan N. and Meyer, Raphael A. and Tropp, Joel A.},
  year = {2025},
  month = aug,
  journal = {arXiv preprint \href{https://arxiv.org/abs/2508.21189v1}{arXiv:2508.21189v1}},
}

@article{SE16,
  author =        {Sorensen, D. C. and Embree, Mark},
  journal =       {SIAM Journal on Scientific Computing},
  month =         jan,
  number =        {3},
  pages =         {A1454-A1482},
  publisher =     {{Society for Industrial and Applied Mathematics}},
  title =         {A {DEIM} Induced {CUR} Factorization},
  volume =        {38},
  year =          {2016},
  doi =           {10.1137/140978430},
}

\end{document}